\def\final{1}
\def\short{0}
\def\submission{0}

\ifnum\short=0

\else

\fi

\ifnum\final=0 
\documentclass[10pt]{article}
\else
\documentclass[11pt]{article}
\fi

\usepackage[titletoc,toc,title]{appendix}
\usepackage{algorithm, algorithmic}
\usepackage{amsmath, amssymb, amsthm, bm}
\usepackage[inline]{enumitem}
\usepackage{framed}
\usepackage{verbatim}
\usepackage{color}
\usepackage{microtype}
\ifnum\short=0 
	\ifnum\submission=0
                \usepackage{kpfonts}
                \DeclareMathAlphabet{\mathsf}{OT1}{cmss}{m}{n}
                \SetMathAlphabet{\mathsf}{bold}{OT1}{cmss}{bx}{n}
                
                \usepackage[margin=1.01in]{geometry}
	\else
		\usepackage{times}
		\usepackage[margin=1.0in]{geometry}
	\fi
\else
	\usepackage{times}
	\usepackage[margin=1.0in]{geometry}
\fi 
\usepackage{multirow, tabularx}
\definecolor{DarkGreen}{rgb}{0.2,0.6,0.2}
\definecolor{DarkRed}{rgb}{0.6,0.2,0.2}
\definecolor{DarkBlue}{rgb}{0.15,0.15,0.55}
\definecolor{DarkPurple}{rgb}{0.4,0.2,0.4}

\usepackage[pdftex]{hyperref}
\hypersetup{
    linktocpage=true,
    colorlinks=true,				
    linkcolor=DarkBlue,		
    citecolor=DarkBlue,	
    urlcolor=DarkBlue,		
}

\ifnum\final = 0
\newcommand{\mynote}[1]{\marginpar{\tiny #1}}
\newcommand{\Bignote}[1]{{\tiny #1}}
\else
\newcommand{\mynote}[1]{}

\newcommand{\Bignote}[1]{}
\fi

\newcommand{\INDSTATE}[1][1]{\STATE\hspace{#1\algorithmicindent}}

\newcolumntype{Y}{>{\centering\arraybackslash}X}

\newcommand{\pr}[2]{\underset{#1}{\mathbb{P}}\left[ #2 \right]}
\newcommand{\ex}[2]{\underset{#1}{\mathbb{E}}\left[ #2 \right]}

\newcommand{\eps}{\varepsilon}

\newcommand{\R}{\mathbb{R}}
\newcommand{\cA}{\mathcal{A}}

\newcommand{\cC}{\mathcal{C}}
\newcommand{\hC}{\hat{\mathcal{C}}}
\newcommand{\cG}{\mathcal{G}}
\newcommand{\cH}{\mathcal{H}}
\newcommand{\hH}{\widehat{\mathcal{H}}}
\newcommand{\cE}{\mathcal{E}}
\newcommand{\cM}{\mathcal{M}}
\newcommand{\cO}{\mathcal{O}}
\newcommand{\cP}{\mathcal{P}}

\newcommand{\cS}{\mathcal{S}}
\newcommand{\cU}{\mathcal{U}}
\newcommand{\cW}{\mathcal{W}}
\newcommand{\cX}{\mathcal{X}}
\newcommand{\cZ}{\mathcal{Z}}

\newtheorem{theorem}{Theorem}[section]

\newtheorem{lem}[theorem]{Lemma}

\newtheorem{claim}[theorem]{Claim}

\theoremstyle{definition}
\newtheorem{defn}[theorem]{Definition}

\newcommand{\dist}{\mathbf{P}}

\newcommand{\bbx}{\mathbf{X}}
\newcommand{\bx}{\mathbf{x}}
\newcommand{\bby}{\mathbf{Y}}
\newcommand{\by}{\mathbf{y}}
\newcommand{\bz}{\mathbf{z}}
\newcommand{\bbz}{\mathbf{Z}}

\newcommand{\adv}{\cA}

\newcommand{\alg}{\cW}
\newcommand{\hf}{\hat{f}}

\newcommand{\queryset}{Q}

\newcommand{\cadv}{\mathsf{Adv}}

\title{Typical Stability}
\author{
\makebox[1.5in]{\hfill Raef Bassily\thanks{University of California
    San Diego, Center for Information Theory and Applications and
             Department of Computer Science and Engineering. 
 \href{mailto:rbassily@ucsd.edu}{rbassily@ucsd.edu}}\hfill}
\and \makebox[1.5in]{\hfill Yoav Freund\thanks{University of California
    San Diego, Department of Computer Science and Engineering. 
 \href{mailto:yfreund@eng.ucsd.edu}{yfreund@eng.ucsd.edu}}\hfill}
}

\date{}

\begin{document}
\maketitle
\pagenumbering{gobble}

\begin{abstract}



In this paper, we introduce a notion of algorithmic stability called \emph{typical stability}. When our goal is to release real-valued queries (statistics) computed over a dataset, this notion does not require the queries to be of bounded sensitivity -- a condition that is generally assumed under differential privacy \cite{DMNS06, Dwork06} when used as a notion of algorithmic stability \cite{DFH1, DFH3, BNS} -- nor does it require the samples in the dataset to be independent -- a condition that is usually assumed when generalization-error guarantees are sought. Instead, typical stability requires the output of the query, when computed on a dataset drawn from the underlying distribution, to be concentrated around its expected value with respect to that distribution. Typical stability can also be motivated as an alternative definition for database privacy. Like differential privacy, this notion enjoys several important properties including robustness to post-processing and adaptive composition. However, privacy is guaranteed only for a given family of distributions over the dataset.

We also discuss the implications of typical stability on the generalization error (i.e., the difference between the value of the query computed on the dataset and the expected value of the query with respect to the true data distribution). We show that typical stability can control generalization error in adaptive data analysis even when the samples in the dataset are not necessarily independent and when queries to be computed are not necessarily of bounded-sensitivity as long as the results of the queries over the dataset (i.e., the computed statistics) follow a distribution with a ``light'' tail. Examples of such queries include, but not limited to, subgaussian and subexponential queries. 



We discuss the composition guarantees of typical stability and prove composition theorems that characterize the degradation of the parameters of typical stability under $k$-fold adaptive composition. We also give simple noise-addition algorithms that achieve this notion. These algorithms are similar to their differentially private counterparts, however, the added noise is calibrated differently.

\end{abstract}
\newpage

\section{Introduction}

Differential privacy \cite{DMNS06, Dwork06} is a strong notion of algorithmic stability that was originally introduced to ensure data privacy. This notion has also been recently re-purposed to control generalization error and ensure statistical validity in adaptive scenarios of data analysis \cite{DFH1, DFH2, DFH3, BNS}. Depsite of its power, the use of differential privacy to control generalization error of real-valued queries (i.e., real-valued statistics) in the aforementioned works entails two main assumptions. The first assumption is that such queries are of bounded sensitivity, that is, the maximum change in the value of the query's output due to a change in any single data point in the dataset has to be bounded. The second assumption is that the samples in the dataset are i.i.d., which is a more standard assumption in the literature when generalization error guarantees are discussed.

\paragraph{Typical Stability.} In this work, we introduce a notion of algorithmic stability called \emph{typical stability} that does not require the queries to be of bounded-sensitivity and does not assume that the samples in the dataset are i.i.d., but instead requires that the output of each query when evaluated on a dataset drawn from the underlying distribution (i.e., the value of the computed statistic) to be ``well concentrated'' around its true mean with respect to that distribution. This notion introduces a new algorithmic-stability approach to controlling generalization error in adaptive data analysis \emph{especially in the settings where queries are not necessarily of bounded sensitivity and the data set entries are not necessarily independent}. Moreover, typical stability can also be motivated as an alternative definition for privacy. Like differential privacy, this notion enjoys several important properties including robustness to post-processing and adaptive composition. However, privacy is guaranteed only for some given family of distributions over the dataset. Our attention in this paper will be devoted to the properties of this notion, its implication on generalization error, and its achievability via simple algorithms. However, we will not discuss any privacy-relared applications of this notion in this paper.


As it is the case with differential privacy, there are two versions of the definition of typical stability: \emph{pure} and \emph{approximate} typical stability. In general, typical stability is defined via three parameters: $\eta, \tau,$ and $\nu$. When $\tau=0$, we call it pure typical stability, otherwise, we call it approximate typical stability. We will give here an intuitive description of this notion (Formal definitions are given in Section~\ref{sec:typical-stab}). Consider a \emph{randomized} algorithm $\adv$ that takes as input a dataset drawn from some arbitrary distribution $\dist$ over the dataset domain $\cX^n$. We say that $\adv$ is $(\eta, \tau, \nu)$-typically stable algorithm if there is a subset $\cS\subseteq \cX^n$ whose measure with respect to $\dist$ is at least $1-\nu$ such that for any pair of datasets $\bx, \by \in\cS$, the distribution of $\adv(\bx)$ is ``close'' to that of $\adv(\by)$. Such closeness is determined by the two parameters $\eta$ and $\tau$ that play similar roles to that of $\eps$ and $\delta$ in differential privacy.  So, intutively, $\cS$ is a set of ``typical'' datasets, and roughly speaking, we require the distributions of the output of $\adv$ on any pair of typical datasets to be ``indistinguishable''.  The intuition is that the output of $\adv$ should ``conceal'' the identity of the dataset inside the typicality set $\cS$, that is, the output should not reveal which of the typical datasets is the true input of $\adv$. However, the output should still reveal information about $\cS$ \emph{as a whole} since such information depends on the underlying distribution $\dist$ rather than the sample. In this sense, typical stability ensures that the whatever is revealed by the algorithm is essentially information shared by all typical datasets.

\paragraph{Concentrated Queries.}  In this work, we consider releasing answers to real-valued queries under typical stability. As mentioned earlier, we consider scenarios where the answers of such queries are concentrated around their expectation with respect to the underlying distribution on the input dataset. We say that a class $\queryset_{\gamma_n}(\dist)$ of real-valued queries on datasets from $\cX^n$ (i.e., datasets of $n$ elements from $\cX$) is $\gamma_n$-concentrated with respect to distribution $\dist$ over $\cX^n$ if there is a non-negative, non-decreasing function $\gamma_n:\mathbb{R}+\rightarrow\mathbb{R}+$ (that possibly depends on $n$) such that for every query $q\in\queryset_{\gamma_n}(\dist)$, with probability at least $1-e^{-\gamma_n(\alpha)}$ over the choice of a dataset $\bbx\sim\dist$, the generalization error for $q$ is bounded by $\alpha$, where generalization error refers to the quantity $\left\vert q(\bbx)-\ex{\bby\sim\dist}{q(\bby)}\right\vert$. For different settings of $\gamma_n$, we obtain, as special cases, query classes such as the class of subgaussian queries and the class of sub-exponential queries which were studied in \cite{zhou15}, and the more special class of jointly Gaussian queries that was studied in both \cite{zhou15, wang-et-al16}. 



\paragraph{Properties of Typical Stability.} We show that typical stability is closed under post-processing. We also show that typical stability is robust to both non-adaptive and adaptive composition (albeit with different guarantees for each case). As the case of adaptive composition is more complicated and requires delicate analysis, we devote Section~\ref{sec:adap-comp} to our results for this case. In particular, we  prove a composition theorem that gives a characterization of the degradation of the parameters of typical stability when a sequence of arbitrarily and adaptively chosen typically stable algorithms\footnote{That is, the choice of each typically stable algorithm possibly depends on all previous outputs and choices of the previous algorithms.} operating on the same dataset are composed together. 

\paragraph{Typical Stability and Generalization.} We also show that typical stability implies generalization by first proving a ``near-independence'' lemma for typically stable algorithms. In particular, we show that with high probability the generalization error for any $\gamma_n$-concentrated query is small as long as the query is generated during an interaction with a typically stable algorithm. In other words, any $\gamma_n$-concentrated query generated via a typically stable interaction with the dataset does not overfit to the same dataset. \emph{As a consequence, typical stability is a rigorous approach to statistcal validity that gives non-trivial guarantees in adaptive settings where queries are not of bounded sensitivity and the data set entries are not independent.}

\paragraph{Achieving Typical Stability via Noise-Adding Mechanisms.} We give simple noise-addition typically stable mechanisms for answering real-valued queries. Our mechanisms are based on adding Laplace and Gaussian noise to the query output to achieve pure and approximate typical stability, respectively. These mechanisms are similar to the differentially private Lalplace and Gaussian mechanisms \cite{DMNS06, DKMMN06}, however, \emph{the added noise is calibrated differently in our case}. In particular, the noise is added based on how ``well'' the query output is concentrated around its expectation, that is, the noise is calibrated to the \emph{confidence interval} of the query's answer over the given dataset. More formally, let $\dist$ be an arbitratry distribution on $\cX^n$. For a query $q:\cX^n\rightarrow \mathbb{R}$, if we have $\pr{X\sim\dist}{\left\vert q(X)- \ex{Y\sim\dist}{q(Y)}\right\vert > \alpha}<\nu$, then, our mechanisms would add noise (Laplace for pure, and Gaussian for approximate typical stability) whose standard deviation proportional to $\alpha$.

\paragraph{Applications in Adaptive Data Analysis.} Our results for this notion have immediate implications on statistical accuracy in adaptive data analysis especially in scenarios where the data samples are not independent and the statistics to be computed are not necessarily of bounded sensitivity. In particular, the established properties of typical stability (mainly its closure under post-processing and adaptive composition, and its generalization guarantees (Sections \ref{sec:properties} and \ref{sec:adap-comp})) together with the simple mechanisms achieving it and their accuracy guarantees (Section \ref{sec:mechanisms}) provide a systematic approach to answering any sequence of adaptively chosen ``concentrated queries'' over a given dataset while ensuring statistical accuracy of all the released answers (with high probability). We note that in scenarios where the samples in the dataset are not independent, a simple approach such as \emph{sample splitting}\footnote{I.e., splitting the dataset randomly into multiple disjoint parts, and using each part to answer a single query.} may not always give statistically accurate answers when the queries are adaptive. Our approach gives rigorous guarantees on generalization error and statistical accuracy in this type of scenarios as long as the computed statistics are reasonably concentrated around their expected value. Examples of such scenarios include those that arise in structured prediction where the dataset entries can be correlated (e.g., nodes of a graph), yet the computed statistics can be well concentrated around their expected values (for example, see \cite{london-et-al-13}).

\paragraph{Other Related Work} Typical stability is a generalization of the notion of \emph{perfect generalization} that was introduced very recently and independently by Cummings et al. \cite{cummings16} where it was also studied in the context of generalization in PAC-learning models. Perfect generalization is a special case of typical stability when the dataset is \emph{i.i.d.~}. Hence, our positive results for the notion of typical stability apply directly to perfect generalization as well. In particular, our positive results on the adaptive composition of typically stable algorithms in Section \ref{sec:adap-comp} imply the same for perfectly generalizing algorithms. 

\noindent Our notion also appears to be similar to the notion of \emph{distributional privacy} that was introduced and briefly studied in \cite{BLR08} as a privacy concept. However, it is imortant to note that typical stability is different from the notion of distributional privacy. In particular, in distributional privacy, the samples in the dataset are assumed to be drawn \emph{without replacement} from the underlying population whereas typical stability does not impose this requirement.

\paragraph{Notation.} We will use the symbol $\cX$ to denote a generic data domain. We will consider data sets of $n \in \mathbb{N}$ samples from $\cX$. We will often use the symbol $\dist$ to denote an arbitrary distribution over $\cX^n$. That is, $\dist$ will in general denote the joint distribution of the samples in the data set. We will use upper-case bold symbols such as $\bbx$ to denote the dataset, that is, $\bbx$ denotes the random sequence $(X_1, \ldots, X_n)$ drawn from $\dist$. Unless stated otherwise, we will use upper-case letters to denote random variables and lower-case letters to denote realizations of random variables. In general, a data analysis algorithm $\adv:\cX^n \rightarrow \cZ$ is an algorithm that takes as input a dataset from $\cX^n$, performs a certain task based on the input dataset, e.g., statistical estimation or learning, and returns some output from the set $\cZ$ (e.g., a parameter estimate, or a classifier). We allow any such algorithm to be randomized, i.e., to have access to its own set random coins, and thus, the output of such algorithm, denoted generically as $Z$, will be a random variable whose randomness depends on both the distribution $\dist$ of the input dataset and the random coins of the algorithm.



\section{Definitions}\label{sec:defs}
Before we formally define typical stability, we first state a standard definition for the notion $(\eta, \tau)$-indistinguishability between distributions of random variables. 

\begin{defn}[$(\eta, \tau)$-indstinguishability]\label{def:indist}
Random variables $X, Y$ with the same range are said to have $(\eta, \tau)$-indistinguishable distributions, denoted as $X\approx_{\eta, \tau}Y$, if for all measurable subsets $\cO$ of their range, we have 
$$\pr{}{X\in\cO}\leq e^{\eta}\pr{}{Y\in\cO}+\tau\quad\text{and}\quad\pr{}{Y\in\cO}\leq e^{\eta}\pr{}{X\in\cO}+\tau.$$ 
\end{defn}

\subsection{Typical Stability}\label{sec:typical-stab}
\begin{defn}[Typical Stability]\label{def:oracle}
Let $\adv:\cX^n\rightarrow \cZ$ be a randomized algorithm. We say that $\adv$ is $(\eta, \tau, \nu)$-typically stable with respect to a family $\cP$ of distributions over $\cX^n$ if for any distribution $\dist\in\cP$ there exists an oracle $\alg$ that takes $\dist$ as input and outputs an element in $\cZ$, such that with probability at least $1-\nu$ over the choice $\bbx\sim\dist$, we have $\adv(\bbx)\approx_{\eta, \tau}\alg(\dist)$.
\end{defn}
When $\tau=0$, the notion will usually be referred to as $(\eta, \nu)$-\emph{pure} typical stability with respect to $\cP$ (as opposed to \emph{approximate} typical stability when $\tau>0$.)

We note that the definition above is almost the same as perfect generalization \cite{cummings16} except that it does not require the samples in the dataset to be i.i.d.~.

Another slightly weaker version of the above definition can be phrased as follows.

\begin{defn}\label{def:typical-stable}
Let $\adv:\cX^n\rightarrow \cZ$ be a randomized algorithm. We say that $\adv$ is $(\eta, \tau, \nu)$-typically stable with respect to a family $\cP$ of distributions over $\cX^n$ if for any distribution $\dist\in\cP$, for any two independent $\bbx, \bby \sim\dist$, with probability at least $1-\nu$ over the choice of $\bbx, \bby$ we have $\adv(\bbx)\approx_{\eta, \tau}\adv(\bby)$.
\end{defn}

It is not hard to see that $(\eta, \tau, \nu)$-typical stability according to Definition~\ref{def:oracle} implies $(2\eta, 3\tau, 2\nu)$-typical stability according to Definition~\ref{def:typical-stable}.

\subsection{Queries}
We define the main class of queries (statistics on the dataset) that will be considered in this paper. We show that several standard classes of statistics are special cases of this class. 

\begin{defn}[\textbf{$\gamma_n$-Concentrated Queries}]\label{def:conc-queries}
For any fixed dataset size $n\in\mathbb{N}$, let $\gamma_n:\mathbb{R}_{+}\rightarrow\mathbb{R}_{+}$ be a non-negative, non-decreasing function (possibly depends on $n$). We define $\queryset_{\gamma_n}(\dist)$ as the class of all real-valued queries defined on $\cX^n$ that when computed on a dataset drawn from the distribution $\dist$ yields an output that is within distance $\alpha$ from its expected value with respect to $\dist$ with probability at least $1-e^{-\gamma_n(\alpha)}$, for every $\alpha>0$. Formally,
$$\queryset_{\gamma_n}(\dist)\triangleq \left\{q:\cX^n\rightarrow\mathbb{R} \text{ s.t. } \pr{\bbx\sim\dist}{\big\vert q(\bbx)- \ex{\bby\sim\dist}{q(\bby)}\big\vert>\alpha}<e^{-\gamma_n(\alpha)}\text{  for all } \alpha>0\right\}.$$ 
\end{defn}

The following are some special cases of $\gamma_n$-concentrated queries:
\begin{itemize}
\item \textbf{$\Delta$-Sensitive Queries:} We let $\queryset_{\Delta}$ denote the class of $\Delta$-sensitive queries on $\cX^n$. That is, 
$$\queryset_{\Delta}=\left\{q:\cX^n\rightarrow\mathbb{R} \text{ s.t. } \forall \bx, \bx'\in\cX^n \text{ with } d_{H}(\bx,\bx')\leq 1, ~\vert q(\bx)-q(\bx')\vert \leq \Delta\right\},$$
where $d_{H}(\cdot, \cdot)$ is the Hamming distance. 

For all functions $\gamma_n$ satisfying $\gamma_n(\alpha)\leq \frac{2\alpha^2}{n\Delta^2}, \alpha>0$, from McDiarmid's inequality, it follows that $\queryset_{\Delta}\subseteq \bigcap\limits_{\dist\in\cP^{\pi}}\queryset_{\gamma_n}(\dist)$ where $\cP^{\pi}$ is the class of all product distributions over $\cX^n$. When $\Delta$ is small, e.g., $\Delta=1/n$, this class is usually referred to as low-sensitivity queries (or, Lipschitz statistics).

\item \textbf{$\sigma$-subgaussian Queries: } A query $q:\cX^n\rightarrow\mathbb{R}$ is said to be $\sigma$-subgaussian with respect to a distribution $\dist$ over $\cX^n$ if we have 
$$\ex{\bbx\sim\dist}{e^{t\left(q(\bbx)-\mu_q\right)}}\leq e^{\frac{1}{2}t^2\sigma^2},~t\in\mathbb{R},$$
where $\mu_q=\ex{\bby\sim\dist}{q(\bby)}$. We denote the class of $\sigma$-subgaussian queries with respect to $\dist$ by $\queryset^{G}_{\sigma}(\dist)$. 

For all functions $\gamma_n$ satisfying $\gamma_n(\alpha)\leq \frac{\alpha^2}{2\sigma^2}, \alpha>0$, we have $\queryset^{G}_{\sigma}(\dist)\subseteq\queryset_{\gamma_n}(\dist)$. Note also that for any $\dist\in\cP^{\pi}$, $\queryset_{\Delta}\subset \queryset^{G}_{\sigma}(\dist)$ for $\Delta\leq\frac{2\sigma}{\sqrt{n}}$.

\item \textbf{$(\sigma, b)$-subexponential Queries: } A query $q:\cX^n\rightarrow\mathbb{R}$ is said to be $(\sigma, b)$-subexponential with respect to a distribution $\dist$ over $\cX^n$ if we have 
$$\ex{\bbx\sim\dist}{e^{t\left(q(\bbx)-\mu_q\right)}}\leq e^{\frac{1}{2}t^2\sigma^2},~\vert t\vert \leq 1/b,$$
where $\mu_q=\ex{\bby\sim\dist}{q(\bby)}$. We denote the class of $(\sigma, b)$-subexponential queries with respect to $\dist$ by $\queryset^{Exp}_{\sigma, b}(\dist)$. 

For all functions $\gamma_n$ satisfying $\gamma_n(\alpha)\leq \min\left(\frac{\alpha^2}{2\sigma^2}, \frac{\sigma^2}{2b^2}\right), \alpha>0$, we have $\queryset^{Exp}_{\sigma, b}(\dist)\subseteq\queryset_{\gamma_n}(\dist)$.

\end{itemize}

\textbf{Note:} We note that there are $\gamma_n$-concentrated queries with respect to distributions that are not product distributions (as in structured prediction where, for example, the dataset can be a graph with correlated nodes, nevertheless, for some types of statistics and under certain assumptions, one can still show that the computed statistics are concentrated in the sense of Definition~\ref{def:conc-queries} (e.g., see \cite[Section 4, Theorem 1]{london-et-al-13} and the references therein).

We also note that there are $\sigma$-subgaussian (or, $(\sigma, b)$-subexponential) queries that are not necessarily $\Delta$-sensitive (not even for arbitrary large $\Delta$), for example, the class of queries that correspond to sums of independent Gaussian random variables, or more generally, independent random variables with bounded moments (but not necessarily bounded support) where concentration inequalities such as Bernstein's can be applied. On the other hand, as indicated above $\Delta$-sensitive queries is a subclass of $\sigma$-subgaussian queries for $\Delta\leq2\sigma/\sqrt{n}$. Hence, the class of subgaussian queries (and more generally the class of $\gamma_n$-concentrated queries) is strictly larger than the class of low-sensitivity queries.

\section{Basic Properties of Typical Stability}\label{sec:properties}

In this section, we discuss some useful properties of typical stability. Together with our results in the following section (Section~\ref{sec:adap-comp}) on the adaptive composition guarantees of typically stable algorithms, these properties play essential role in establishing the guarantees of typical stability on the generalization error in adaptive data analysis.

\subsection{Closure under post-processing and non-adaptive composition}
First, we discuss some useful properties of typically stable algorithms. In particular, we show that this notion is closed under post-processing and under non-adaptive composition where the stability parameters degrade linearly with the number of algorithms to be composed. We defer the discussion of the adaptive composition of typically stable algorithms to Section~\ref{sec:adap-comp}.

\begin{lem}[Postprocessing]\label{lem:postpro}
Let $\adv:\cX^n\rightarrow \cZ$ be an $(\eta, \tau, \nu)$-typically stable algorithm with respect to a family of distributions $\cP$ over $\cX^n$. Let $\mathcal{B}:\cZ^n\rightarrow \cU$ be a randomized algorithm. Let $\mathcal{C}:\mathcal{X}^n\rightarrow\cU$ be defined as $\mathcal{C}(\bx)=\mathcal{B}(\adv(\bx))$. Then, $\mathcal{C}$ is $(\eta, \tau, \nu)$-typically stable algorithm with respect to $\cP$.   
\end{lem}
\begin{proof}
The proof follows from a straightforward manipulation where the probability of any measurable subset $\cO$ of the outcomes of $\mathcal{C}$ is expressed as a convex combination of probabilities of a collection of sets over the outcomes of $\adv$.
\end{proof}

The next theorem characterizes the degradation in privacy parameters under non-adaptive composition of typically stable algorithms.
\begin{lem}[Non-Adaptive Composition]\label{lem:non-adap-comp}
Let $\adv_i:\cX^n\rightarrow \cZ_i, ~ i\in [k]$ be any collection of $(\eta, \tau, \nu)$-typically stable algorithms with respect to a common family $\cP$ of distributions over $\cX^n$ where all $\adv_i, i\in [k]$ have independent random coins. Define $\mathcal{C}:\cX^n\rightarrow\cZ_1\times\ldots\times\cZ_k$ as $\mathcal{C}(\bx)=\left(\adv_1(\bx), \ldots, \adv_k(\bx)\right), x\in\cX^n$. Then, $\mathcal{C}$ is $(k\eta, k\tau, k\nu)$-typically stable. 
\end{lem}
\begin{proof}
The proof follows from combining the union bound with a standard technique for bounding the joint probability of the output of $\mathcal{C}$ similar to that used in proving basic composition theorem for $(\eps, \delta)$-differentially private algorithms \cite{DKMMN06, dp-text}.
\end{proof}


\subsection{Typical stability and the Near-Independence property} 

The following lemma describes an important implication of typical stability. It shows the impact of typical stability on the \emph{joint distribution} of the input and output of a typically stable algorithm. Let $\cP$ be a family of distributions over $\cX^n$. Let $\dist$ be a distribution in $\cP$ and $\bbx\sim\dist$. Let $\adv:\cX^n\rightarrow\cZ$ be an $(\eta, \tau, \nu)$-typically stable algorithm with respect to $\cP$. Lemma~\ref{max-inf-bound} below states that the \emph{joint} probability measure of $\left(\bbx, \adv(\bbx)\right)$ is ``close'' to the \emph{product} measure of $\bbx$ and $\adv(\bbx)$ (i.e., the measure induced by the product of the marginal distributions of $\bbx$ and $\adv(\bbx)$). A statement of the same spirit is known for $\eps$-differentially private algorithms via a connection to the notion of max-information \cite{DFH2} (and more recently for $(\eps, \delta)$-differentially private algorithms when their inputs are drawn from a product distribution \cite{max-inf-dp-16}). 



\begin{lem}[Near-Independence Lemma of Typically Stable Algorithms]\label{max-inf-bound}
Let $\cP$ be a family of distributions over $\cX^n$, and $\adv:\cX^n\rightarrow \cZ$ be an $(\eta, \tau, \nu)$-typically stable algorithm with respect to $\cP$. Let $\dist\in\cP$, and $\bbx, \bby \sim\dist$ be two independent random variables. Let $\cS\subseteq\cX^n$ be a set that satisfies the condition of typical stability in Definition~\ref{def:typical-stable}, that is, $\pr{\bbx, \bby\sim\dist}{\bbx, \bby\in\cS}\geq 1-\nu$ and for every $\bx, \by\in\cS$,  we have $\adv(\bx)\approx_{\eta, \tau} \adv(\by)$. Then, for every measurable $\cO\subseteq \cX^n\times\cZ$, we have
\begin{align}
\pr{}{\left(\bbx, \adv(\bbx)\right)\in\cO~\big\vert~\bbx\in\cS}&\leq e^{\eta}\pr{}{\left(\bbx, \adv(\bby)\right)\in\cO~\big\vert~\bbx, \bby\in\cS} + \tau\label{ineq:cond-indep}
\end{align} 
Moreover, if $\eta< 1$ and $\nu< 1/10$, then (\ref{ineq:cond-indep}) implies 
\begin{align}
\pr{}{\left(\bbx, \adv(\bbx)\right)\in\cO}&\leq e^{\eta}\pr{}{\left(\bbx, \adv(\bby)\right)\in\cO}+\tau+5\nu\label{ineq:indep}
\end{align}

\end{lem}
\begin{proof}
Fix $\cO\subseteq\cX^n\times\cZ^n$. For every $\bx\in\cX^n$, let $U_{\bx}=\{z\in\cZ^n: (\bx, z)\in\cO\}$. Now, observe\footnote{For continuous measures, we will regard the probabilities below as density functions and sums are replaced with Lebesgue integrals with respect to the appropriate probability measures.} 
\begin{align}
\pr{}{(\bbx, \adv(\bbx))\in\cO~\big\vert~\bbx\in\cS}&=\sum_{\bx\in\cX^n}\pr{}{\adv(\bx)\in U_x~\big\vert~ \bbx=\bx, \bbx\in\cS}\pr{}{\bbx=\bx~\big\vert \bbx\in\cS}\nonumber\\
&\leq \sum_{\bx\in\cX^n}e^{\eta}\pr{}{\adv(\by)\in U_{\bx}}\pr{}{\bbx=\bx~\big\vert \bbx\in\cS} + \tau\nonumber
\end{align}
for every $\by\in\cS$. The last inequality follows from the typical stability of $\adv$. Now, by taking the expectation of the two sides of the above inequality with respect to the conditional measure $\pr{\cdot\leftarrow\by}{\bby=\cdot~\big\vert~\bby\in\cS}$, we get (\ref{ineq:cond-indep}). 

Now, from (\ref{ineq:cond-indep}) and by straightforward manipulation, we have 
\begin{align}
\pr{}{\left(\bbx, \adv(\bbx)\right)\in\cO}&\leq e^{\eta}\pr{}{\left(\bbx, \adv(\bby)\right)\in\cO~\big\vert~\bby\in\cS} + \tau + \nu\nonumber\\
&\leq e^{\eta}\frac{\pr{}{\left(\bbx, \adv(\bby)\right)\in\cO}}{\pr{}{\bby\in\cS}} + \tau + \nu\nonumber\\
&\leq e^{\eta}\pr{}{\left(\bbx, \adv(\bby)\right)\in\cO} + \tau + 5\nu\nonumber
\end{align}
where the last inequality follows from the fact that $\pr{}{\bby\in\cS}\geq 1-\nu$, and that $\eta<1, \nu<1/10$. 
\end{proof}

%


\subsection{Generalization via Typical Stability}\label{subsec:gen-error}
We discuss here an important implication of Lemma~\ref{max-inf-bound}. Let $\dist$ be a distribution over $\cX^n$. The next theorem states that if a query in $\queryset_{\gamma_n}(\dist)$ is generated as a result of an $(\eta, \tau, \nu=e^{-\gamma_n(\alpha)})$-typically stable algorithm with input dataset $\bbx\sim\dist$, then the generlaization error for that query on the dataset will be bounded by $\alpha$ with high probability.

\begin{theorem}[Generalization via Typical Stability]\label{thm:gen-error}
Let $\dist$ be any distribution on $\cX^n$. Let \mbox{$\adv:\cX^n\rightarrow\queryset_{\gamma_n}(\dist)$} be an $(\eta, \tau, \nu)$-typically stable algorithm with respect to $\dist$ that outputs a $\gamma_n$-concentrated query in $\queryset_{\gamma_n}(\dist)$. Let $q_{\bbx}$ denote the output of $\adv(\bbx)$. Let $\alpha=\inf\left\{r\in\R_{+}:~e^{-\gamma_n(\alpha)}\leq \nu\right\}$. Then, we have 
\begin{align}
\pr{}{\big\vert q_{\bbx}(\bbx)-\ex{\mathbf{T}\sim\dist}{q_{\bbx}(\mathbf{T})} \big\vert > \alpha}&\leq \left(e^{\eta}+5\right)\nu+\tau\nonumber
\end{align}
In particular, if $\eta=O(1)$, then
\begin{align}
\pr{}{\big\vert q_{\bbx}(\bbx)-\ex{\mathbf{T}\sim\dist}{q_{\bbx}(\mathbf{T})} \big\vert > \alpha}&\leq O\left(\max\left(\nu, \tau\right)\right)\nonumber
\end{align}
\end{theorem}

\begin{proof}
The proof follows from Lemma~\ref{max-inf-bound}. In particular, if we define the event $\cO$ in  Lemma~\ref{max-inf-bound} as
$$\cO=\left\{ (\bx, q_{\bx}): \left\vert q_{\bx}(\bx) - \ex{\mathbf{T}\sim\dist}{q_{\bx}(\mathbf{T})} \right\vert > \alpha \right\},$$ 
then by (\ref{ineq:indep}) in Lemma~\ref{max-inf-bound}, the definition of $\alpha$, and the fact that the output of $\adv$ is $\gamma_n$-concentrated query with respect to $\dist$, we get the desired result.
\end{proof}

\section{Adaptive Composition of Typically Stable Algorithms}\label{sec:adap-comp}

In this section, we discuss an important property of typical stability. We give a characterization of how fast typical stability degrades as a result of adaptively composing $(\eta, \tau, \nu)$-typically stable algorithms. 


Before we state our composition results, we first describe the adaptive composition model.

Let $\cP$ be a family of distribrutions over $\cX^n$. We consider an arbitrary sequence of $k$ adaptively chosen algorithms $\adv_i:\cX^n\times\cZ_1\times\ldots\times\cZ_{i-1}\rightarrow\cZ_i,~i\in[k]$ such that for every $i\in[k]$ and every $\bz^{i-1}\triangleq (z_1,\ldots, z_{i-1})\in\cZ_1\times\ldots\times\cZ_{i-1}$, the algorithm $\adv_i(\cdot, \bz^{i-1})$ is $(\eta, \tau, \nu)$-typically stable with respect to $\cP$. We consider the $k$-fold adaptive composition mechanism $\cM^k$ outlined in Algorithm~\ref{alg:comp-mech}. 

\begin{algorithm}[htb]
	\caption{$\cM^k:\cX^n\rightarrow \cZ_1\times\ldots\times\cZ_k$: A $k$-fold adaptive composition mechanism for typically stable algorithms}
	\begin{algorithmic}[1]
\REQUIRE {A dataset $\bx\in\cX^n$ and a composition adversary $\cadv$.}
\STATE {Initialize $\bz^0=\bot$, i.e., empty string.}
\FOR {$i=1,\dots, k$}
\STATE {$\cadv$ chooses an $(\eta, \tau, \nu)$-typically stable algorithm (w.r.t. $\cP$) $\adv_i:\cX^n\times\cZ_1\times\ldots\times\cZ_{i-1}\rightarrow \cZ_i$.}
\STATE {$\cadv$ receives $z_i= \adv_i(\bx, \bz^{i-1})$.}
\ENDFOR
\RETURN  $\bz^k=(z_1,\ldots, z_k)$.
	\end{algorithmic}
	\label{alg:comp-mech}
\end{algorithm}

\begin{defn}
We say that the class of $(\eta, \tau, \nu)$-typically stable mechanisms w.r.t. a family of distributions $\cP$ satisfies $(\eta', \tau', \nu')$-typical stability w.r.t. $\cP$ under $k$-fold adaptive composition if mechanism $\cM^k$ (Algorithm~\ref{alg:comp-mech}) is $(\eta', \tau', \nu')$-typically stable w.r.t. $\cP$.
\end{defn}

\subsection{Composition of pure typically stable algorithms}\label{subsec:proof-comp-pure}
In this section, we state and prove our composition theorem for pure typically stable algorithms, i.e., when $\tau=0$. Whenever we refer to the composition mechanism $\cM^k$ in this section, it will be assumed that $\cM^k$ runs with $\tau=0$.

\begin{theorem}[Adaptive Composition of Pure Typically Private Algorithms]\label{thm:composition-pure}
Let $k\geq 2$. For all $\eta >0$, $0\leq \nu < 1$, and $0<\tau' < 1$, the class of $(\eta, 0, \nu)$-typically stable algorithms w.r.t. $\cP$ satisfies $(\eta^*, \tau^*, \nu^*)$-typical stability w.r.t. $\cP$ under $k$-fold adaptive composition where $\eta^*, \tau^*$, and $\nu^*$ are given by
\begin{align}
\eta^*&= 3\sqrt{2k\log(1/\tau')}\eta+3k\eta\left(e^{\eta}-1\right),\nonumber\\
\tau^*&=\nu^*= 5\left(k\tau'/\eta+\nu/\eta+\sum_{t=1}^{k-1}e^{\eta_{t}}\nu/\eta\right)^{1/2}\nonumber
\end{align}
In particular, for $\eta=O\left(\sqrt{\frac{\log(1/\tau')}{k}}\right)$, the $k$-fold adaptive composition is 
$$\left(\tilde{O}\left(\sqrt{k}\eta\right),~ O\left(\sqrt{k\tau'/\eta}+\sqrt{k\nu/\eta}e^{\tilde{O}\left(\sqrt{k}\eta\right)}\right), O\left(\sqrt{k\tau'/\eta}+\sqrt{k\nu/\eta}e^{\tilde{O}\left(\sqrt{k}\eta\right)}\right)\right)-\text{typically stable}$$ where the hidden logarithmic factor in the $\tilde{O}(\cdot)$ expressions is $\approx \sqrt{\log(1/\tau')}.$
\end{theorem}

We note here that $\nu^*$ does not scale linearly with $k$ as one would expect if a simple application of the union bound would have been used. A straightforward application of the union bound would not be appropriate in an adaptive setting since, at each step of the $k$-fold composition, conditioning on the previous outputs effectively changes the data distribution. 


The high-level idea of the proof is as follows. Suppose $\bbx\sim\dist$. Let $\bbz^i=(Z_1, \ldots, Z_i)$ denote the output of $\cM^i(\bbx)$ and $\tilde{\bbz}^i=(\tilde{Z}_1, \ldots, \tilde{Z}_i)$ denote the output of $\alg^i(\dist)=(\alg_1(\dist), \ldots, \alg_i(\dist))$ where $\alg_i(\dist)$ is the oracle corresponding to $\adv_i(\cdot, \bbz^{i-1})$ as described in Definition~\ref{def:oracle}. At each step $i$ of the composition, we prune the bad set of pairs $(\bbx, \bbz^{i-1})$ for which the distribution of the output of $\adv_i(\bbx, \bbz^{i-1})$ is not $\eta$-indistinguishable from that of $\alg_i(\dist)$ (the oracle corresponding to $\adv_i(\cdot, \bbz^{i-1})$). By doing so at each step $i=1, \ldots, j$, then at step $j$, we are left with a good set for which this indistinguishability condition holds for all $i\in [j]$. We then use a standard concentration inequality (Azuma's inequality; see Theorem~\ref{thm:azuma}) to argue that by removing another tiny portion (of small measure) from that good set, we can ensure that the joint distribution of $\left(\bbx, \cM^j(\bbx)\right)$ is $\eta_j$-indistinguishable from the joint distribution of $\left(\bbx, \alg^j(\dist)\right)$ for some $\eta_j$ that will be determined shortly ($\eta_j$ is given by (\ref{eta_j2}) below). We then apply a useful lemma from \cite{KS14} (see Lemma~\ref{lem:condition}) to argue typical stability of $\cM^j$. In our proof, in order to bound the measure of the bad set as we go from one step of the composition to the next, we use induction. Lemma~\ref{lem:ind-step} serves as the induction step where we bound the additional ``bad'' measure we need to prune as we go from step $j$ to step $j+1$. Similar idea of pruning ``bad'' sets and applying the \emph{conditioning lemma} of \cite{KS14} appeared in \cite[Proof of Theorem 3.1]{max-inf-dp-16}, albeit in a different context and without involving an induction argument. 

\subsubsection{Proof of Theorem~\ref{thm:composition-pure}}\label{subsec:proof-comp-pure}

Fix some $\dist\in\cP$. Let $\bbx\sim\dist$. For any integer $j\geq 1$, we use $\bbz^j=(Z_1, \ldots, Z_j)$ to denote the output of $\cM^j(\bbx)$. We will also use $\tilde{\bbz}^j$ to denote the output of $\alg^j(\dist)$. In the sequel, we will assume that $\bbz^{0}=\tilde{\bbz}^{0}=\bot$ with probability 1.  

\noindent Let 
\begin{align}
\eta_{j}&= \sqrt{2j\log(1/\tau')}\eta+j\eta\left(e^{\eta}-1\right),\label{eta_j2}\\
\tau_{j}&=\nu_{j}= \left(j\tau'/\eta+\nu/\eta+\sum_{t=1}^{j-1}e^{\eta_{t}}\nu/\eta\right)^{1/2}\label{tau_j2}
\end{align}

\noindent For any integer $i\geq 1$, and any $(\bx, \bz^i)\in\cX^n \times \cZ_1\times\ldots\times\cZ_i$, we define 
\begin{align}
f_i(\bx, \bz^i)&=\ln\left(\frac{\pr{}{\adv_i\left(\bx, \bz^{i-1}\right)=z_i}}{\pr{}{\alg_i\left(\dist\right)=z_i}}\right)\label{def:f_i}
\end{align}
We then define 
\begin{align}
F_j(\bx, \bz^j)&=\sum_{i=1}^j f_i(\bx, \bz^i)\label{def:F_j}
\end{align}
and
\begin{align}
\hf_i(\bx, \bz^{i-1})&=\max\limits_{z\in\cZ_i}\left\vert f_i\left(\bx, (\bz^{i-1}, z)\right)\right\vert \label{def:hf_i}
\end{align}
where $(\bz^{i-1}, z)\in \cZ_1\times\ldots\times\cZ_{i-1}\times\cZ_i$. We define the sequence of sets
\begin{align}
\cC_i&=\left\{(\bx, \bz^{i-1}):~ \hf_i(\bx, \bz^{i-1})\leq \eta\right\}, \label{def:C_i}
\end{align}
and 
\begin{align}
\hC_i&=\left\{(\bx, \bz^{i-1}): ~\forall \ell\in[i] ~(\bx, \bz^{\ell-1})\in\cC_{\ell} \right\}\label{def:hC_i}
\end{align}
Fix $\tau'>0$. Let $\eta_j$ be as defined in (\ref{eta_j2}). Now, we define the set of ``good inputs and outputs'' as 
\begin{align}
\cG_j&=\left\{(\bx, \bz^j):~F_j(\bx, \bz^j)\leq \eta_j\right\}\label{def:G_j}
\end{align}

\begin{lem}\label{lem:basic}
Let $j\geq 1$. Suppose there is $\mu^*\geq j\nu$ such that $\pr{}{(\bbx, \bbz^{j-1})\in\hC_j}\geq 1-\mu^*$. Then, 
$$\pr{}{(\bbx, \bbz^j)\notin\cG_j\text{  and  }(\bbx, \bbz^{j-1})\in\hC_j}\leq \tau',$$ and hence, 
$$\pr{}{(\bbx, \bbz^j)\in\cG_j}\geq 1-\mu^*-\tau'.$$
\end{lem}
\begin{proof}

Consider the random variables $f_i(\bbx, \bbz^i),~i\in [j]$. Fix $i\in [j]$. Given $(\bx, \bz^{i-1})\in\hC_i$, then with probability $1$, we have 
\begin{align}
\max\limits_{\zeta\in\cZ_i}\left\vert f_i(\bx, \bz^{i-1}, \zeta)\right\vert\leq \eta, \label{azuma-cond-1}
\end{align}
Now, given $(\bx, \bz^{i-1})\in\hC_i$, we analyze the conditional expectation $\ex{}{f_i(\bbx, \bbz^i)~\vert~\bx, \bz^{i-1}}$ $\big($where we use the notation $\ex{}{f_i(\bbx, \bbz^i)~\vert~\bx, \bz^{i-1}}$ as a shorthand for $\ex{}{f_i(\bbx, \bbz^i)~\vert~\bbx=\bx, ~\bbz^{i-1}=\bz^{i-1}}\big)$. We note that this is the conditional KL-divergence between the distributions of the outputs of $\adv_i$ and $\alg_i$ conditioned on $(\bx, \bz^{i-1})$. Using a standard bound on this KL-divergence (see \cite{DRV10}\footnote{\cite{DR16} gives a tighter bound by a factor of $1/2$.}) together with (\ref{azuma-cond-1}), we have 
\begin{align}
\ex{}{f_i(\bbx, \bbz^i)~\vert~\bx,~\bz^{i-1}}\leq \eta(e^{\eta}-1)\label{azuma-cond-2}
\end{align} 
Now, define 
\begin{align}
g_i(\bx, \bz^{i})&=\left\{ \begin{array}{ll}
         f_i(\bx, \bz^i) & \mbox{for $(\bx, \bz^{i-1})\in\hC_i$};\\
        0 & \mbox{otherwise}.\end{array} \right.\nonumber 
\end{align}

Observe that using (\ref{azuma-cond-1}), we have $\left\vert g_i(\bbx, \bbz^{i})\right\vert\leq \eta$ with probability $1$. Moreover, using (\ref{azuma-cond-2}), for all $(\bx, \bz^{i-1})\in\cX\times\cZ_1\times\ldots\times\cZ_{i-1},$ we have 
$$\ex{}{g_i(\bbx, \bbz^{i})~\vert~\bx, \bz^{i-1}}\leq  \eta(e^{\eta}-1).$$

Now, we use the following classical concentration inequality. 

\begin{theorem}[Azuma's Inequality]\label{thm:azuma}
Let $T_1, \ldots, T_j$ be a sequence of random variables such that for every $i\in[j]$ we have 
$$\pr{}{\vert T_i\vert\leq b}=1$$ 
and for every prefix $\mathbf{T}^{i-1}=\mathbf{t}^{i-1}$ we have 
$$\ex{}{T_i~\vert~\mathbf{t}^{i-1}}\leq c$$
then for all $u\geq 0$, we have
$$\pr{}{\sum\limits_{i=1}^j T_i\geq j c+ u\sqrt{j} b}\leq e^{-u^2/2}$$
\end{theorem}

Now, let $\eta_j$ be as defined in (\ref{eta_j2}). Observe that 
\begin{align}
\pr{}{(\bbx, \bbz^j)\notin\cG_j\text{  and  }(\bbx, \bbz^{j-1})\in\hC_j}&= \pr{}{\sum_{i=1}^j f_i(\bbx, \bbz^j)>\eta_j\text{  and  } (\bbx, \bbz^{j-1})\in\hC_j}\nonumber\\
&=\pr{}{\sum_{i=1}^j g_i(\bbx, \bbz^j)>\eta_j} \leq\tau'\nonumber
\end{align}
where the last inequality follows from Theorem~\ref{thm:azuma}. This together with the premise in the lemma concludes the proof.

\end{proof}

\begin{lem}\label{lem:typ-stable}
Let $j\geq 1$. Suppose the premise of Lemma~\ref{lem:basic} is true, that is, there is $\mu^*\geq j\nu$ such that $\pr{}{(\bbx, \bbz^{j-1})\in\hC_j}\geq 1-\mu^*$. Then $\cM^j$ is $(3\eta_j,~ 5\sqrt{\frac{\mu^*+\tau'}{\eta}},~ 5\sqrt{\frac{\mu^*+\tau'}{\eta}})$-typically stable.
\end{lem}

\begin{proof}
Recall the definition of $\cG_j$ in (\ref{def:G_j}). Observe that for any $(\bx, \bz^{j})\in\cG_j$
\begin{align}
\frac{\pr{}{\cM^j(\bx)=\bz^j}}{\pr{}{\alg^j(\dist)=\bz^j}}&\leq e^{\eta_j}\label{direct1}
\end{align}
where $\alg^j(\dist)=\left(\alg_1(\dist),\ldots, \alg_j(\dist)\right)$. Let $\tilde{\bbz}^j$ denote the output of $\alg^j(\dist)$. Define
$$\tilde\cG_j=\left\{(\bx, \bz^j):~-F_j(\bx, \bz^j)\leq \eta_j\right\}$$
where $F_j(\bx, \bz^j)$ is as defined in (\ref{def:F_j}). Now, for any $(\bx, \bz^{j})\in\tilde\cG_j$, we also have
\begin{align}
\frac{\pr{}{\alg^j(\dist)=\bz^j}}{\pr{}{\cM^j(\bx)=\bz^j}}&\leq e^{\eta_j}\label{direct2}
\end{align}
Moreover, by the independence of $\bbx$ and $\tilde{\bbz}^{j-1}$, using the union bound we get
$$\pr{}{(\bbx,\tilde{\bbz}^{j-1})\in\hC_j}\geq 1-j\nu$$
where $\hC_j$ is as defined in (\ref{def:hC_i}). 

Thus, by swaping the roles of $\cM^j(\bbx)$ and $\alg^j(\dist)$ in Lemma~\ref{lem:basic}, it follows that 
\begin{align}
\pr{}{(\bbx, \tilde{\bbz}^{j})\in\tilde\cG_j}\geq 1-j\nu-\tau'\geq 1-\mu^*-\tau'\label{same-lemma}
\end{align}
Hence, by Lemma~\ref{lem:basic}, and using (\ref{direct1})-(\ref{same-lemma}) above, we have 
\begin{align}
\left(\bbx,~\cM^j(\bbx)\right)&\approx_{\eta_j,~ \mu^*+\tau'}\left(\bbx,~\alg^j(\dist)\right)\nonumber
\end{align}
Now, we use the following useful lemma from \cite{KS14}. 
\begin{lem}[Conditioning Lemma \cite{KS14}]\label{lem:condition}
Suppose that $(U, V)\approx_{\epsilon, \delta}(U', V')$. Then, for every $\hat\delta>0$, the following holds
\begin{align}
\pr{t\sim p(V)}{U\vert_{V=t}\approx_{3\epsilon, \hat\delta}U'\vert_{V'=t}}&\geq 1- \frac{2\delta}{\hat\delta}-\frac{2\delta}{1-e^{-\epsilon}}\nonumber
\end{align}
where $p(V)$ denotes the distribution of $V$.
\end{lem}
Now, by replacing the pairs $(U, V)$ and $(U', V')$ in Lemma~\ref{lem:condition} with $\left(\cM^j(\bbx), \bbx\right)$ and $\left(\alg^j(\dist), \bbx\right)$, respectively, and also replacing $\epsilon, ~\delta$ with $\eta_j,~\mu^*+\tau'$, respectively, we get 
$$\pr{\bbx\sim\dist}{\cM^j(\bbx)\approx_{3\eta_j,~\hat\delta}\alg^j(\dist)}\geq 1- \frac{2(\mu^*+\tau')}{\hat\delta}-\frac{2(\mu^*+\tau')}{1-e^{-\eta_j}}$$ 
for every $\hat\delta>0.$ We conclude the proof by setting $\hat\delta=2\sqrt{\left(\mu^*+\tau'\right)\eta/5}$ and noting that $1-e^{-\eta_j}\geq 1-e^{-\eta}\geq \frac{2}{5}\min(\eta, 1)$.
\end{proof}

The following lemma serves as the induction step in our proof.
\begin{lem}\label{lem:ind-step}
Let $j\geq 1$. Suppose the premise of Lemma~\ref{lem:basic} is true, i.e., there exists $\mu^*\geq j\nu$ such that $\pr{}{(\bbx, \bbz^{j-1})\in\hC_j}\geq 1-\mu^*$. Then, we have 
$$\pr{}{(\bbx, \bbz^j)\in \hC_{j+1}}\geq 1- \mu^* - e^{\eta_j}\nu -\tau'$$
\end{lem}
\begin{proof}
From the definitions of the sequence of sets $\cC_i$, $\hC_i$, and $\cG_i$ in (\ref{def:C_i}), (\ref{def:hC_i}), and (\ref{def:G_j}), respecitvely, observe that 
\begin{align}
\pr{}{(\bbx, \bbz^j)\notin \hC_{j+1}}=&\pr{}{(\bbx, \bbz^j)\notin \cC_{j+1}\text{  and  } (\bbx, \bbz^{j-1})\in\hC_{j}}+\pr{}{(\bbx, \bbz^{j-1})\notin\hC_{j}}\nonumber\\
=&\pr{}{(\bbx, \bbz^j)\notin \cC_{j+1}\text{  and  } (\bbx, \bbz^{j-1})\in\hC_{j} \text{  and  }(\bbx, \bbz^j)\in\cG_j}\nonumber\\
&+\pr{}{(\bbx, \bbz^j)\notin \cC_{j+1}\text{  and  } (\bbx, \bbz^{j-1})\in\hC_{j} \text{  and  }(\bbx, \bbz^j)\notin\cG_j}+\pr{}{(\bbx, \bbz^{j-1})\notin\hC_{j}}\nonumber\\
\leq&\pr{}{(\bbx, \bbz^j)\notin \cC_{j+1}\text{  and  }(\bbx, \bbz^j)\in\cG_j}+\pr{}{(\bbx, \bbz^{j-1})\in\hC_{j} \text{  and  }(\bbx, \bbz^j)\notin\cG_j}+\pr{}{(\bbx, \bbz^{j-1})\notin\hC_{j}}\nonumber\\
\leq&\pr{}{(\bbx, \bbz^j)\notin \cC_{j+1}\text{  and  }(\bbx, \bbz^j)\in\cG_j}+\tau'+\mu^*\nonumber
\end{align}
where the last inequality follows from Lemma~\ref{lem:basic} and the fact that the premise of Lemma~\ref{lem:basic} is true. Now, consider the remaining term. Let $\cC^c_{j+1}$ denote the complement of the set $\cC_{j+1}$, and let $\tilde{\bbz}^{j}$ denote the output of $\alg^j(\dist)$. Observe that 
\begin{align}
\pr{}{(\bbx, \bbz^j)\notin \cC_{j+1}\text{  and  }(\bbx, \bbz^j)\in\cG_j}&=\sum_{(\bx, \bz^j)\in \cC^c_{j+1}\cap\cG_j}\pr{}{\bbx=\bx, \bbz^j=\bz^j}\nonumber\\
&=\sum_{(\bx, \bz^j)\in \cC^c_{j+1}\cap\cG_j}\pr{}{\bbz^j=\bz^j~\vert~\bbx=\bx}\pr{}{\bbx=\bx}\nonumber\\
&=\sum_{(\bx, \bz^j)\in \cC^c_{j+1}\cap\cG_j}\pr{}{\cM^j(\bx)=\bz^j}\pr{}{\bbx=\bx}\nonumber\\
&\leq e^{\eta_j}\sum_{(\bx, \bz^j)\in \cC^c_{j+1}\cap\cG_j}\pr{}{\alg^j(\dist)=\bz^j}\pr{}{\bbx=\bx}\nonumber\\
&\leq e^{\eta_j}\pr{}{(\bbx, \tilde{\bbz^j})\notin \cC_{j+1}}\nonumber
\end{align}
where the fourth inquality follows from the definition of $\cG_j$. 

Notice that $\bbx$ and $\tilde{\bbz}^j$ are independent. By the $(\eta, 0, \nu)$-typical stability of $\adv_{j+1}$, for any fixed prefix $\bz^{j}$ we have
$$\pr{}{\adv_{j+1}(\bbx, \bz^j)\approx_{\eta}\alg_{j+1}(\dist)}\geq 1-\nu$$
Thus, by the independence of $\bbx$ and $\tilde{\bbz}^j$, we have
$$\pr{}{\adv_{j+1}(\bbx, \tilde{\bbz^j})\approx_{\eta}\alg_{j+1}(\dist)}\geq 1-\nu,$$
i.e.,
$$\pr{}{(\bbx, \tilde{\bbz^j})\notin \cC_{j+1}}\leq \nu.$$ This concludes the proof.
\end{proof}

The proof of Theorem~\ref{thm:composition-pure} follows from the above three lemmas and by induction on the basis of $j=1$. Note that $\pr{}{(\bbx, \bbz^{0})\in\hC_1}=\pr{}{(\bbx, \bot)\in\hC_1}\geq 1-\nu$ by the typical stability of $\adv_1$, and hence the premise in Lemma~\ref{lem:basic} is true for $j=1$ (the base case of the induction). 

\subsection{Composition of approximate typically stable algorithms}\label{subsec:proof-comp-approx}
Our composition theorem for approximate typically stable algorithms, i.e. when $\tau>0$, is given by Theorem~\ref{thm:composition} below. The proof of this theorem follows the general paradigm of the proof of the pure case in Section~\ref{subsec:proof-comp-pure}. However, the proof requires few non-trivial modifications to account for another sequence of ``bad'' sets that arises from the fact that $\tau>0$. We defer the details of this proof to the appendix.


\begin{theorem}[Adaptive Composition of Approximate Typically Stable Algorithms]\label{thm:composition}
Let $k\geq 2$. For all $\eta \in (0, 3/2]$, $\tau\in (0, \eta/50]$, $\nu \in (0, 1)$, and $\tau'\in (0, 1)$, the class of $(\eta, \tau, \nu)$-typically stable algorithms w.r.t. a family of distributions $\cP$ satisfies $(\eta^*, \tau^*, \nu^*)$-typical stability w.r.t. $\cP$ under $k$-fold adaptive composition where $\eta^*, \tau^*$, and $\nu^*$ are given by
\begin{align}
\eta^*&= 6\sqrt{2k\log(1/\tau')}\eta+3k\left(2\eta\left(\frac{e^{2\eta}}{1-\hat{\tau}}-1\right)+\psi(\tau)\right),\nonumber\\
\tau^*&= \nu^*= 5\left(k\frac{\hat{\tau}+\tau'}{2\eta}+\frac{\nu}{2\eta}+\sum_{t=1}^{k-1}e^{\eta_{t}}\frac{\nu}{2\eta}\right)^{1/2}\nonumber
\end{align}
where 
\begin{align}
&\hspace{4.5cm}\hat{\tau}\triangleq\frac{2\tau}{1-e^{-\eta}}=O\left(\tau/\eta\right),\nonumber\\
\psi(\tau)&\triangleq\tau \left(2e^{\eta}+1\right)+\tau^2\left(1+\frac{2e^{2\eta}}{(e^{\eta}-1)^2}\left(4e^{2\eta}+4e^{\eta}-3-2e^{-\eta}+e^{-2\eta}\right)\right)=O\left(\tau+\tau^2/\eta^2\right).\nonumber
\end{align}

In particular, for $\tau=O(\eta^2)$, and $\eta=O\left(\sqrt{\frac{\log(1/\tau')}{k}}\right)$, the $k$-fold adaptive composition is 
$$\left(\tilde{O}\left(\sqrt{k}\eta\right),~ O\left(\sqrt{k\left(\tau/\eta^2+\tau'/\eta\right)}+\sqrt{k\nu/\eta}e^{\tilde{O}\left(\sqrt{k}\eta\right)}\right), O\left(\sqrt{k\left(\tau/\eta^2+\tau'/\eta\right)}+\sqrt{k\nu/\eta}e^{\tilde{O}\left(\sqrt{k}\eta\right)}\right)\right)-\text{typically stable}$$ where the hidden logarithmic factor in the $\tilde{O}(\cdot)$ expressions is $\approx \sqrt{\log(1/\tau')}.$
\end{theorem}

\section{Typically Stable Algorithms for Answering Concentrated Queries}\label{sec:mechanisms}

We describe here two simple noise-adding mechanims for answering real-valued $\gamma_n$-concentrated queries while achieving typical stability. The algorithms are based on adding Laplace or Gaussian noise to the output of the query, and hence these algorithms are very similar to their differntially private counterparts, however, the added noise is calibrated differently from the case of differential privacy. 

Rather than calibrating the noise based on the global sensitivity \cite{DMNS06} (or smoothed local sensitivity \cite{NRS07}) of the query, the noise added here is proportional to the width of the \emph{confidence interval} of the query's answer w.r.t. the given dataset, that is, based on how well the value of the statistic computed over the dataset is concentrated around its true mean with respect to the underlying distribution on the data. In particular, for the class of $\gamma_n$-concentrated queries, to achieve typical stability with parameters $\eta$ and $\nu=O\left(e^{-\gamma_n(\alpha)}\right)$, the noise added is proportional to $\alpha$, namely, the magnitude (standard deviation) of the noise is chosen to be $\approx \alpha/\eta$. 

The intution here is that the added noise aims to hide the identity of the specific sample given to the algorithm (i.e., the input dataset) so that it ``blends'' with the other \emph{typical} datasets that occur with probability greater than $1-\nu$. At the same time, the noisy output would still reveal only the relevant information that is shared by all typical datasets. Such information depends more on the distribution rather than the sample. 



Let $\dist$ be any distribution over $\cX^n$. Algorithm~\ref{alg:laplace} describes a $(\eta, 0, \nu)$-typically stable algorithm for answering $\gamma_n$-concentrated queries with respect to $\dist$, i.e.,  for answering any query in $\queryset_{\gamma_n}(\dist)$. 

%
%

\begin{algorithm}[htb]
	\caption{A Pure Typically Stable Laplace Mechanism for answering queries in $\queryset_{\gamma_n}(\dist)$}
	\begin{algorithmic}[1]
		\REQUIRE {A dataset $\bx\in\cX^n$, parameters $\eta, \nu$, and a $\gamma_n$-concentrated query $q\in\queryset_{\gamma_n}(\dist)$ (the function $\gamma_n$ is also given as input  to the algorithm).}

\STATE{Choose $\alpha$ such that $\gamma_n(\alpha)\geq\ln\left(1/\nu\right).$}
\STATE{Let $N\sim \mathsf{Lap}\left(\frac{\alpha}{\eta}\right)$.}
\RETURN  $w=q(\bx)+N$.
	\end{algorithmic}
	\label{alg:laplace}
\end{algorithm}

\begin{theorem}[Typical Stability of the Algorithm~\ref{alg:laplace}]\label{thm:typ-stab-lap}
For any distribution $\dist$ over $\cX^n$, Algorithm~\ref{alg:laplace} for answering queries in $\queryset_{\gamma_n}(\dist)$ is $(\eta, 0, \nu)$-typically stable with respect to $\dist$.
\end{theorem}

\begin{proof}
Let $\cS^{q}\triangleq\left\{\bx\in\cX^n: \big\vert q(\bx)-\ex{\mathbf{T}\sim\dist}{q(\mathbf{T})}\big\vert \leq \alpha\right\}$ where $q$ is the input query from $\queryset_{\gamma_n}(\dist)$. Let $\alg$ be the oracle that takes the distribution $\dist$ as input, computes $\ex{\mathbf{T}\sim\dist}{q(\mathbf{T})}$, then adds $\mathsf{Lap}\left(\frac{\alpha}{\eta}\right)$ noise to it and outputs the result. By definition of $\queryset_{\gamma_n}$, we have $\pr{\bbx\sim\dist}{\bbx\in\cS^{q}}\geq 1-e^{-\gamma_n(\alpha)}\geq 1-\nu$. 

Let $\bx \in \cS^{q}$. Let $w$ be the output of Algorithm~\ref{alg:laplace} for input dataset $\bx$, and $\cO\subseteq \mathbb{R}$ be any measurable set. Observe 
\begin{align}
\pr{}{w\in\cO}=\pr{}{q(\bx)+\mathsf{Lap}\left(\frac{\alpha}{\eta}\right)\in\cO}&=\frac{\eta}{2\alpha}\int_{\cO}e^{-\frac{\eta}{\alpha}\vert w-q(\bx)\vert}dw\nonumber\\
&\leq e^{\frac{\eta}{\alpha}\big\vert q(\bx)-\ex{\mathbf{T}\sim\dist}{q(\mathbf{T})}\big\vert}\cdot\frac{\eta}{2\alpha}\int_{\cO}e^{-\frac{\eta}{\alpha}\big\vert w- \ex{\mathbf{T}\sim\dist}{q(\mathbf{T})}\big\vert}dw\nonumber\\
&\leq e^{\eta} \pr{}{\ex{\mathbf{T}\sim\dist}{q(\mathbf{T})}+\mathsf{Lap}\left(\frac{\alpha}{\eta}\right)\in\cO}\nonumber
\end{align}
where the last inequality follows from the fact that $\big\vert q(\bx)-\ex{\mathbf{T}\sim\dist}{q(\mathbf{T})}\big\vert \leq \alpha$ since $\bx\in\cS^{q}$. 

\end{proof}

The following theorem follows directly from the tail properties of Laplace distribution. 
\begin{theorem}[Empirical Error of Algorithm~\ref{alg:laplace}]\label{thm:emp-error-lap}
With probability at least $1-\beta$ with respect to the random coins of Algorithm~\ref{alg:laplace}, the output $w$ satisfies $\vert w-q(\bx)\vert< \frac{\alpha\ln(1/\beta)}{\eta}$.
\end{theorem}

Next, we describe and analyze a mechanism for answering $\gamma_n$-concentrated queries while satisfying approximate typical stability.

Let $\dist$ be any distribution over $\cX^n$. Algorithm~\ref{alg:gauss} describes a $(\eta, \tau, \nu)$-typically stable algorithm for answering $\gamma_n$-concentrated queries with respect to $\dist$.


\begin{algorithm}[htb]
	\caption{An Approximate Typically Stable Gaussian Mechanism for answering queries in $\queryset_{\gamma_n}(\dist)$}
	\begin{algorithmic}[1]
		\REQUIRE {A dataset $\bx\in\cX^n$, parameters $\eta, \tau, \nu$, and a $\gamma_n$-concentrated query $q\in\queryset_{\gamma_n}(\dist)$ (the function $\gamma_n$ is also given as input  to the algorithm).}

\STATE{Choose $\alpha$ such that $\gamma_n(\alpha)\geq\ln\left(1/\nu\right).$}
\STATE{Let $N\sim \mathcal{N}\left(0,\sigma^2\right)$ where $\sigma=\frac{\alpha\sqrt{2\ln(1.5/\tau)}}{\eta}$.\label{step:add-gauss}}
\RETURN  $w=q(\bx)+N$.
	\end{algorithmic}
	\label{alg:gauss}
\end{algorithm}

\begin{theorem}[Typical Stability of Algorithm~\ref{alg:gauss}]\label{thm:approx-stable-gauss}
For any distribution $\dist$ over $\cX^n$, Algorithm~\ref{alg:gauss} for answering queries in $\queryset_{\gamma_n}(\dist)$ is $(\eta, \tau, \nu)$-typically stable with respect to $\dist$.
\end{theorem}
\begin{proof}
As in the proof of Theorem~\ref{thm:typ-stab-lap}, let $\cS^{q}\triangleq\left\{\bx\in\cX^n: \big\vert q(\bx)-\ex{\mathbf{T}\sim\dist}{q(\mathbf{T})}\big\vert \leq \alpha\right\}$ where $q$ is the input query from $\queryset_{\gamma_n}(\dist)$. Let $\alg$ be the oracle that takes the distribution $\dist$ as input, computes $\ex{\mathbf{T}\sim\dist}{q(\mathbf{T})}$, then adds $\mathcal{N}\left(0,\sigma^2\right)$ noise to it and outputs the result, where $\sigma$ is as given in Step~\ref{step:add-gauss} of Algorithm~\ref{alg:gauss}. By definition of $\queryset_{\gamma_n}$, we have $\pr{\bbx\sim\dist}{\bbx\in\cS^{q}}\geq 1-e^{-\gamma_n(\alpha)}\geq 1-\nu$.

Let $\bx \in \cS^{q}$. To reach the desired result, we then bound the ratio of the densities of $q(\bx)+N$ and $\ex{\mathbf{T}\sim\dist}{q(\mathbf{T})} + N$ where $N\sim \mathcal{N}\left(0,\sigma^2\right)$. This is done by following the same approach in the analysis of the standard Gaussian mechanism in the literature of differential privacy \cite{DKMMN06, dp-text, NTZ12} and using the fact that $\big\vert q(\bx) - \ex{\mathbf{T}\sim\dist}{q(\mathbf{T})}\big\vert \leq \alpha$.

\end{proof}

The following theorem follows directly from the tail properties of the Gaussian distribution.
\begin{theorem}[Empirical Error of Algorithm~\ref{alg:gauss}]\label{thm:emp-error-approx}
With probability at least $1-\beta$ with respect to the random coins of Algorithm~\ref{alg:gauss}, the output $w$ satisfies $\vert w-q(\bx)\vert< \frac{2\alpha\sqrt{\ln(1.5/\tau)\ln(1/\beta)}}{\eta}$.
\end{theorem}



\section{Applications in Adaptive Data Analysis}

In this section, we discuss a direct implication of our results in previous sections on the statistical accuracy resulting from using a typically stable algorithm to answer any sequence of $k$ adaptively chosen $\gamma_n$-concentrated queries. 

In particular, we give a theorem (Theorem~\ref{thm:acc-adap} below) that describes an upper bound on the worst true error in the answers of Algorithm~\ref{alg:laplace} (of Section~\ref{sec:mechanisms}) to any sequence of $k$ adaptively chosen $\gamma_n$-concentrated queries from the class $\queryset_{\gamma_n}(\dist)$. Similar guarantees can be also obtained for Algorithm~\ref{alg:gauss}.  

First, we describe the adaptive model that we consider here. Figure~\ref{proc:adap} describes an adaptive interaction between an analyst and Algorithm~\ref{alg:laplace}. Let $\dist$ be a distribution over $\cX^n$. The procedure consists of $k$ iterations where in each iteration $j\in[k]$, the analyst chooses a query $q_j$ adaptively from $\queryset_{\gamma_n}(\dist)$ (based on all the previously submitted queries $q_1, \ldots, q_{j-1}$ and answers $w_1, \ldots, w_{j-1}$ received from the algorithm), the analyst submits $q_j$ to Algorithm~\ref{alg:laplace}, then the algorithm returns an answer $w_j$. 

\begin{figure}[ht!]
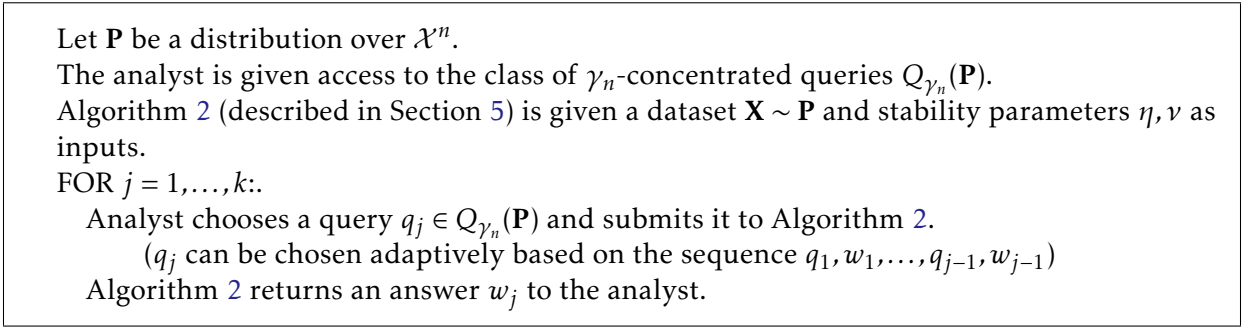

\begin{framed}
\begin{algorithmic}
\STATE{Let $\dist$ be a distribution over $\cX^n$.}
\STATE{The analyst is given access to the class of $\gamma_n$-concentrated queries $\queryset_{\gamma_n}(\dist)$.}
\STATE{Algorithm~\ref{alg:laplace} (described in Section~\ref{sec:mechanisms}) is given a dataset $\bbx\sim\dist$ and stability parameters $\eta, \nu$ as inputs.}
\STATE{FOR $j=1,\dots, k$:.}
\INDSTATE[1]{Analyst chooses a query $q_j \in \queryset_{\gamma_n}(\dist)$ and submits it to Algorithm~\ref{alg:laplace}.} 
\INDSTATE[3]{($q_j$ can be chosen adaptively based on the sequence $q_1, w_1, \ldots, q_{j-1}, w_{j-1}$)}
\INDSTATE[1]{Algorithm~\ref{alg:laplace} returns an answer $w_j$ to the analyst.}
\end{algorithmic}
\vspace{-2mm}
\end{framed}
\vspace{-6mm}
\caption{Adaptive interaction between an analyst and Algorithm~\ref{alg:laplace} \label{proc:adap}}
\end{figure}

The worst \emph{true} error is defined as the maximum statistcal error in all the answers of the algorithm. Given the adaptive procedure in Figure~\ref{proc:adap}, the worst true error is formally defined as $\max\limits_{j\in [k]}\big\vert \ex{\bby\sim\dist}{q_j(\bby)} - w_j\big\vert$.  

The following theorem characterizes an upper bound on such an error. This theorem is a direct consequence of the generalization guarantees in Section~\ref{subsec:gen-error} (Theorem~\ref{thm:gen-error}), the adaptive composition guarantees of typical stability in Section~\ref{sec:adap-comp} (Theorem~\ref{thm:composition-pure}), and the guarantees of Algorithm~\ref{alg:laplace} in Section~\ref{sec:mechanisms} (Theorems~\ref{thm:typ-stab-lap} and \ref{thm:emp-error-lap}). 

\begin{theorem}\label{thm:acc-adap}
Let $\dist$ be a distribution over $\cX^n$. Let $\alpha >0,~0 <\beta<1$. Let $k\in\mathbb{N}$. Consider the adaptive procedure for answering $k$ adaptively chosen queries from $\queryset_{\gamma_n}(\dist)$ as described in Figure~\ref{proc:adap}. Suppose Algorithm~\ref{alg:laplace} is run with input parameters $\eta=O\left(\frac{1}{\sqrt{k\ln(k/\beta)}}\right)$ and $\nu = O\left(e^{-\gamma_n\left(\alpha/\sqrt{k}\ln^{3/2}(k/\beta)\right)}\right)$. Then, we have
\begin{align}
\pr{}{\max\limits_{j\in [k]}\bigg\vert \ex{\bby\sim\dist}{q_j(\bby)} - w_j\bigg\vert\geq \alpha}&= O\left(k^{3/4}\ln^{1/4}(k/\beta)e^{-\Omega\left(\gamma_n\left(\alpha/\sqrt{k}\ln^{3/2}(k/\beta)\right)\right)}+\beta\right)\nonumber
\end{align}
where the probability is over the randomness in the choice of $q_1, \ldots, q_k$ (i.e., the possible randomness in the choice of the analyst) and over the randomness in $w_1, \ldots, w_k$ (i.e., randomness in Algorithm~\ref{alg:laplace} and in the choice of its input $\bbx\sim\dist$).
\end{theorem}

\section*{Acknowledgement}
We would like to thank Adam Smith and James Zhou for the insightful discussions. We also thank Adam Smith for pointing out the conditioning lemma of \cite{KS14} which we used in our proof for the adaptive composition guarantees of typical stability. 

\ifnum\short=1
\vfill\eject \small
\bibliographystyle{alpha}
\bibliography{references}
\else
\addcontentsline{toc}{section}{References}
\bibliographystyle{alpha}
\bibliography{references}
\fi

\begin{appendices}
\section{Proof of Theorem~\ref{thm:composition}}

Fix $\eta \in (0, 3/2]$, $\tau\in (0, \eta/50]$, $\nu \in (0, 1)$, and $\tau'\in (0, 1)$. For any $j\geq 1$, let 
\begin{align}
\eta_{j}&= 2\sqrt{2j\log(1/\tau')}\eta+j\left(2\eta\left(\frac{e^{2\eta}}{1-\hat{\tau}}-1\right)+\psi(\tau)\right),\label{eta_j}\\
\tau_{j}&= \nu_j= \left(j\frac{\hat{\tau}+\tau'}{2\eta}+\frac{\nu}{2\eta}+\sum_{t=1}^{j-1}e^{\eta_{t}}\frac{\nu}{2\eta}\right)^{1/2}.\label{tau_j}
\end{align}
where 
\begin{align}
&\hspace{4.5cm}\hat{\tau}\triangleq\frac{2\tau}{1-e^{-\eta}}=O\left(\tau/\eta\right),\label{def:hat-tau}\\
\psi(\tau)&\triangleq\tau \left(2e^{\eta}+1\right)+\tau^2\left(1+\frac{2e^{2\eta}}{(e^{\eta}-1)^2}\left(4e^{2\eta}+4e^{\eta}-3-2e^{-\eta}+e^{-2\eta}\right)\right)=O\left(\tau+\tau^2/\eta^2\right).\label{def:psi-tau}
\end{align}

Fix some $\dist\in\cP$. Let $\bbx\sim\dist$. As before, for any integer $j\geq 1$, we let $\bbz^j=(Z_1, \ldots, Z_j)$ denote the output of $\cM^j(\bbx)$, and let $\tilde{\bbz}^j$ to denote the output of $\alg^j(\dist)$. We assume that $\bbz^{0}=\tilde{\bbz}^{0}=\bot$ with probability 1.  

Analogous to the definitions in the case of pure typical stability, for any integer $i\geq 1$, and any $(\bx, \bz^i)\in\cX^n \times \cZ_1\times\ldots\times\cZ_i$, we define 
\begin{align}
f_i(\bx, \bz^i)&=\ln\left(\frac{\pr{}{\adv_i\left(\bx, \bz^{i-1}\right)=z_i}}{\pr{}{\alg_i\left(\dist\right)=z_i}}\right)\label{def:f_i2}
\end{align}
and 
\begin{align}
F_j(\bx, \bz^j)&=\sum_{i=1}^j f_i(\bx, \bz^i)\label{def:F_j2}
\end{align}

\noindent Unlike the case of pure typical stability, here we define different sequence of sets $\cC_i$:
\begin{align}
\cC_i&=\left\{(\bx, \bz^{i-1}):~ \adv_i(\bx, \bz^{i-1})\approx_{\eta, \tau} \alg_i\left(\dist\right)\right\}, \label{def:C_i2}
\end{align}

\noindent We also define the sequence of sets of input/output pairs for which $\vert f_i(\bx, \bz^i)\vert$ is bounded:
\begin{align}
\cE_i&=\left\{(\bx, \bz^i):~ \vert f_i(\bx, \bz^i)\vert \leq 2\eta\right\}\label{def:E_i}
\end{align}

\noindent Now, we define the sequences $\cH_i$ and $\hH_i$ as follows:
\begin{align}
\cH_i&=\left\{(\bx, \bz^{i}): ~\forall \ell\in[i] ~(\bx, \bz^{\ell-1})\in\cC_{\ell} \text{  and  } (\bx, \bz^{\ell})\in\cE_{\ell}\right\}\label{def:H_i2}\\
\hH_i&=\left\{(\bx, \bz^{i-1}): ~\forall \ell\in[i] ~(\bx, \bz^{\ell-1})\in\cC_{\ell} \text{  and  } (\bx, \bz^{\ell-1})\in\cE_{\ell-1}\right\}\label{def:hH_i2}
\end{align}
where $\cE_0\triangleq \cX^n\times \{\bot\}.$

\noindent Let $\eta_j$ be as defined in (\ref{eta_j}), and define the set of ``good inputs and outputs'' as 
\begin{align}
\cG_j&=\left\{(\bx, \bz^j):~F_j(\bx, \bz^j)\leq \eta_j\right\}\label{def:G_j2}
\end{align}

We start by a lemma that will play a similar role to that of Lemma~\ref{lem:basic} in the proof of Theorem~\ref{thm:composition-pure}. However, we note that both the statement and the proof of this lemma are different from those of Lemma~\ref{lem:basic}.
\begin{lem}\label{lem:basic2}
Let $j\geq 1$ and $\hat\tau$ be as defined in (\ref{def:hat-tau}). Suppose there exists a $\mu^*\geq j(\nu+\hat\tau)$ such that $\pr{}{(\bbx, \bbz^{j-1})\in\hH_j}\geq 1-\mu^*$. Then, 
$$\pr{}{(\bbx, \bbz^j)\notin\cG_j\text{  and  }(\bbx, \bbz^{j-1})\in\hH_j}\leq \tau'+\hat\tau,$$ and 
$$\pr{}{(\bbx, \bbz^j)\in\cG_j}\geq 1-\mu^*-\tau'-\hat\tau.$$
\end{lem}

The proof of the above lemma relies mainly on the following claims. The proof of these claim follow ideas similar to those in \cite[Claim 3.5 and Lemma 3.7]{max-inf-dp-16}.

\begin{claim}\label{claim-a}
Let $\eta \in (0, 3/2]$ and $\tau\in (0, \eta/50]$. Let $(\bx, \bz^{i-1})\in \hH_i$. We have
$$\ex{}{f_i(\bx, \bz^i)\big\vert~\bbx=\bx,~ \bbz^{i-1}=\bz^{i-1},~(\bbx, \bbz^{i})\in\cH_i}\leq 2\eta\left(\frac{e^{2\eta}}{1-\hat{\tau}}-1\right)+\psi(\tau)$$
where $\hat\tau$ and $\psi(\tau)$ are as defined in (\ref{def:hat-tau}) and (\ref{def:psi-tau}), respectively.
\end{claim}

\begin{claim}\label{claim-b}
Let $(\bx, \bz^{i-1})\in \hH_i$ and $\hat\tau$ be as defined in (\ref{def:hat-tau}). Then, we have

$$\pr{}{(\bbx, \bbz^i)\notin\cH_i\big\vert \bbx=\bx,~\bbz^{i-1}=\bz^{i-1}} < \hat\tau$$
\end{claim}

\noindent\textbf{Proof of Lemma~\ref{lem:basic2}}\\
Define 
\begin{align}
g_i(\bx, \bz^{i})&=\left\{ \begin{array}{ll}
         f_i(\bx, \bz^i) & \mbox{for $(\bx, \bz^{i})\in\cH_i$};\\
        0 & \mbox{otherwise}.\end{array} \right.\nonumber 
\end{align}
First, by definition, we have
\begin{align}
\pr{}{\big\vert g_i(\bx, \bz^{i})\big\vert \leq 2\eta}&=1\label{eq:prob1}
\end{align}
Next, we claim that for all $(\bx, \bz^{i-1})\in \cX^n\times \cZ_1\times\ldots\times\cZ_{i-1}$
\begin{align}
\ex{}{g_i(\bx, \bz^{i})\big\vert ~\bbx=\bx,~ \bbz^{i-1}=\bz^{i-1}}&\leq 2\eta\left(\frac{e^{2\eta}}{1-\hat{\tau}}-1\right)+\psi(\tau) \label{eq:expec}
\end{align}
To see this, first, observe that for any $(\bx, \bz^{i-1})\in \hH_i,$ we have 
\begin{align}
\ex{}{g_i(\bx, \bz^{i})\big\vert ~\bbx=\bx,~ \bbz^{i-1}=\bz^{i-1}, (\bbx, \bbz^{i})\notin\cH_i}&=0.\label{bound-expec1}
\end{align}
Moreover, for any $(\bx, \bz^{i-1})\in \hH_i,$ by using Claim~\ref{claim-a}, we get
\begin{align}
\ex{}{g_i(\bx, \bz^{i})\big\vert ~\bbx=\bx,~ \bbz^{i-1}=\bz^{i-1},~(\bbx, \bbz^{i})\in\cH_i}&=\ex{}{f_i(\bx, \bz^i)\big\vert~\bbx=\bx,~ \bbz^{i-1}=\bz^{i-1},~(\bbx, \bbz^{i})\in\cH_i}\nonumber\\
&\leq 2\eta\left(\frac{e^{2\eta}}{1-\hat{\tau}}-1\right)+\psi(\tau)\label{bound-expec2}
\end{align}
For any $(\bx, \bz^{i-1})\notin \hH_i$, we have $(\bbx, \bbz^{i})\notin\cH_i$ with probability $1$, and hence 
\begin{align}
\ex{}{g_i(\bx, \bz^{i})\big\vert ~\bbx=\bx,~ \bbz^{i-1}=\bz^{i-1}}=0\label{bound-expec3}
\end{align}
Thus, from (\ref{bound-expec1}), (\ref{bound-expec2}), and (\ref{bound-expec3}), we have shown that (\ref{eq:expec}) is true for all $(\bx, \bz^{i-1})\in \cX^n\times \cZ_1\times\ldots\times\cZ_{i-1}$. 

Now, given (\ref{eq:prob1}) and (\ref{eq:expec}), we can apply Azuma's inequality (Theorem~\ref{thm:azuma}) to the sequence of random variables $g_i(\bx, \bz^{i}), ~i\in [j]$. Let $\eta_j$ be as defined in (\ref{eta_j}). Observe that 
\begin{align}
\pr{}{(\bbx, \bbz^j)\notin\cG_j\text{  and  }(\bbx, \bbz^{j})\in\cH_j}&= \pr{}{\sum_{i=1}^j f_i(\bbx, \bbz^j)>\eta_j\text{  and  } (\bbx, \bbz^{j})\in\cH_j}\nonumber\\
&=\pr{}{\sum_{i=1}^j g_i(\bbx, \bbz^j)>\eta_j} \leq\tau'\label{apply-azuma}
\end{align}
where the last inequality follows from Azuma's inequality (Theorem~\ref{thm:azuma}) and the definition of $\eta_j$ in (\ref{eta_j}). Now, we have 
 \begin{align}
&\pr{}{(\bbx, \bbz^j)\notin\cG_j\text{  and  }(\bbx, \bbz^{j-1})\in\hH_j}\nonumber\\
&= \pr{}{(\bbx, \bbz^j)\notin\cG_j\text{  and  }(\bbx, \bbz^{j-1})\in\hH_j \text{  and  }(\bbx, \bbz^{j})\in\cH_j}+ \pr{}{(\bbx, \bbz^j)\notin\cG_j\text{  and  }(\bbx, \bbz^{j-1})\in\hH_j \text{  and  }(\bbx, \bbz^{j})\notin\cH_j}\nonumber\\
&\leq\pr{}{(\bbx, \bbz^j)\notin\cG_j\text{  and  }(\bbx, \bbz^{j})\in\cH_j} + \pr{}{(\bbx, \bbz^{j-1})\in\hH_j \text{  and  }(\bbx, \bbz^{j})\notin\cH_j}\nonumber\\
&\leq \pr{}{(\bbx, \bbz^j)\notin\cG_j\text{  and  }(\bbx, \bbz^{j})\in\cH_j} + \pr{}{(\bbx, \bbz^{j})\notin\cH_j~\big\vert (\bbx, \bbz^{j-1})\in\hH_j}\nonumber\\
&\leq \tau'+\hat\tau\nonumber
\end{align}
where the last inequality follows from (\ref{apply-azuma}) and Claim~\ref{claim-b}. This together with the fact that $\pr{}{(\bbx, \bbz^{j-1})\in\hH_j}\geq 1-\mu^*$ (the premise of Lemma~\ref{lem:basic2}) gives 
$$\pr{}{(\bbx, \bbz^j)\in\cG_j}\geq 1-\mu^*-\tau'-\hat\tau.$$ Hence, the proof of Lemma~\ref{lem:basic2} is complete.

Next, we state and prove an analog of Lemma~\ref{lem:typ-stable}. 

\begin{lem}\label{lem:typ-stable2}
Let $j\geq 1$. Suppose the premise of Lemma~\ref{lem:basic2} is true, that is, there is $\mu^*\geq j(\nu+\hat\tau)$ such that $\pr{}{(\bbx, \bbz^{j-1})\in\hH_j}\geq 1-\mu^*$. Then $\cM^j$ is $(3\eta_j,~ 5\sqrt{\frac{\mu^*+\tau'+\hat\tau}{\eta}},~ 5\sqrt{\frac{\mu^*+\tau'+\hat\tau}{\eta}})$-typically stable.
\end{lem}

\begin{proof}
Recall the definition of $\cG_j$ in (\ref{def:G_j2}). Observe that for any $(\bx, \bz^{j})\in\cG_j$
\begin{align}
\frac{\pr{}{\cM^j(\bx)=\bz^j}}{\pr{}{\alg^j(\dist)=\bz^j}}&\leq e^{\eta_j}\label{direct12}
\end{align}
where $\alg^j(\dist)=\left(\alg_1(\dist),\ldots, \alg_j(\dist)\right)$. Let $\tilde{\bbz}^j$ denote the output of $\alg^j(\dist)$. Define
$$\tilde\cG_j=\left\{(\bx, \bz^j):~-F_j(\bx, \bz^j)\leq \eta_j\right\}$$
where $F_j(\bx, \bz^j)$ is as defined in (\ref{def:F_j2}). Now, for any $(\bx, \bz^{j})\in\tilde\cG_j$, we also have
\begin{align}
\frac{\pr{}{\alg^j(\dist)=\bz^j}}{\pr{}{\cM^j(\bx)=\bz^j}}&\leq e^{\eta_j}\label{direct22}
\end{align}
Moreover, by the independence of $\bbx$ and $\tilde{\bbz}^j$, using Claim~\ref{claim-b} and the union bound, we get
$$\pr{}{(\bbx,\tilde{\bbz}^{j-1})\in\hH_j}\geq 1-j(\nu+\hat\tau)\geq 1-\mu^*$$
where $\hH_j$ is as defined in (\ref{def:hH_i2}). 

Thus, by swaping the roles of $\cM^j(\bbx)$ and $\alg^j(\dist)$ in Lemma~\ref{lem:basic2}, it follows that 
\begin{align}
\pr{}{(\bbx, \tilde{\bbz}^{j})\in\tilde\cG_j}\geq 1-\mu^*-\tau'-\hat\tau\label{same-lemma2}
\end{align}
Hence, by Lemma~\ref{lem:basic2}, and using (\ref{direct12})-(\ref{same-lemma2}) above, we have 
\begin{align}
\left(\bbx,~\cM^j(\bbx)\right)&\approx_{\eta_j,~ \mu^*+\tau'+\hat\tau}\left(\bbx,~\alg^j(\dist)\right)\nonumber
\end{align}

The rest of the proof follows by applying Lemma~\ref{lem:condition} in the same way as in the proof of Lemma~\ref{lem:typ-stable}.

\end{proof}

The final component of the proof is the following lemma that is analogous to Lemma~\ref{lem:ind-step} and serves as our induction step. 

\begin{lem}\label{lem:ind-step2}
Let $j\geq 1$. Suppose the premise of Lemma~\ref{lem:basic2} is true, i.e., there exists $\mu^*\geq j(\nu+\hat\tau)$ such that $\pr{}{(\bbx, \bbz^{j-1})\in\hH_j}\geq 1-\mu^*$. Then, we have 
$$\pr{}{(\bbx, \bbz^j)\in \hH_{j+1}}\geq 1- \mu^* - e^{\eta_j}\nu -\tau'-\hat\tau$$
\end{lem}
\begin{proof}
From the definitions of the sequence of sets $\hH_j, \cC_j,$ and $\cE_j$ in (\ref{def:hH_i2}), (\ref{def:C_i2}), and (\ref{def:E_i}), respecitvely, observe that 
\begin{align}
&\pr{}{(\bbx, \bbz^j)\notin \hH_{j+1}}\nonumber\\
=&\pr{}{(\bbx, \bbz^j)\notin \hH_{j+1}\text{  and  } (\bbx, \bbz^{j-1})\in\hH_{j}}+\pr{}{(\bbx, \bbz^{j-1})\notin\hH_{j}}\nonumber\\
=&\pr{}{(\bbx, \bbz^{j})\notin\cC_{j+1} \text{ and }(\bbx, \bbz^{j})\in \cH_j}+\pr{}{(\bbx, \bbz^j)\notin\cE_j\text{ and }(\bbx, \bbz^{j-1})\in\hH_j}+\pr{}{(\bbx, \bbz^{j-1})\notin\hH_{j}}\nonumber\\
=&\pr{}{(\bbx, \bbz^{j})\notin\cC_{j+1} \text{ and }(\bbx, \bbz^{j})\in \cH_j}+\pr{}{(\bbx, \bbz^j)\notin\cH_j\text{ and }(\bbx, \bbz^{j-1})\in\hH_j}+\pr{}{(\bbx, \bbz^{j-1})\notin\hH_{j}}\nonumber\\
\leq& \pr{}{(\bbx, \bbz^{j})\notin\cC_{j+1} \text{ and }(\bbx, \bbz^{j})\in \cH_j}+\pr{}{(\bbx, \bbz^j)\notin\cH_j\text~\big\vert~(\bbx, \bbz^{j-1})\in\hH_j}+\pr{}{(\bbx, \bbz^{j-1})\notin\hH_{j}}\nonumber\\
\leq &\pr{}{(\bbx, \bbz^{j})\notin\cC_{j+1} \text{ and }(\bbx, \bbz^{j})\in \cH_j}+\hat\tau+\mu^*\label{ineq-last1}
\end{align}
where the last inequality follows from Claim~\ref{claim-b} and from the premise in the lemma. 

Now, let's consider the first term on the right-hand side of the last inequality. Observe that
\begin{align}
&\pr{}{(\bbx, \bbz^{j})\notin\cC_{j+1} \text{ and }(\bbx, \bbz^{j})\in \cH_j}\nonumber\\
=& \pr{}{(\bbx, \bbz^{j})\notin\cC_{j+1} \text{ and }(\bbx, \bbz^{j})\in \cH_j \text{  and  }(\bbx, \bbz^j)\in\cG_j}+ \pr{}{(\bbx, \bbz^{j})\notin\cC_{j+1} \text{ and }(\bbx, \bbz^{j})\in \cH_j \text{  and  }(\bbx, \bbz^j)\notin\cG_j}\nonumber\\
\leq& \pr{}{(\bbx, \bbz^{j})\notin\cC_{j+1}\text{  and  }(\bbx, \bbz^j)\in\cG_j}+ \pr{}{(\bbx, \bbz^{j})\in \cH_j \text{  and  }(\bbx, \bbz^j)\notin\cG_j}\nonumber\\
\leq& \pr{}{(\bbx, \bbz^{j})\notin\cC_{j+1} \text{  and  }(\bbx, \bbz^j)\in\cG_j}+ \tau'\label{ineq-last2}
\end{align}
where the last inequality follows from (\ref{apply-azuma}). Now, consider the remaining term. Let $\cC^c_{j+1}$ denote the complement of the set $\cC_{j+1}$, and let $\tilde{\bbz}^{j}$ denote the output of $\alg^j(\dist)$. Observe that 
\begin{align}
\pr{}{(\bbx, \bbz^{j})\notin\cC_{j+1}\text{  and  }(\bbx, \bbz^j)\in\cG_j}&=\sum_{(\bx, \bz^j)\in \cC^c_{j+1}\cap\cG_j}\pr{}{\bbx=\bx, \bbz^j=\bz^j}\nonumber\\
&=\sum_{(\bx, \bz^j)\in \cC^c_{j+1}\cap\cG_j}\pr{}{\bbz^j=\bz^j~\vert~\bbx=\bx}\pr{}{\bbx=\bx}\nonumber\\
&=\sum_{(\bx, \bz^j)\in \cC^c_{j+1}\cap\cG_j}\pr{}{\cM^j(\bx)=\bz^j}\pr{}{\bbx=\bx}\nonumber\\
&\leq e^{\eta_j}\sum_{(\bx, \bz^j)\in \cC^c_{j+1}\cap\cG_j}\pr{}{\alg^j(\dist)=\bz^j}\pr{}{\bbx=\bx}\nonumber\\
&\leq e^{\eta_j}\pr{}{(\bbx, \tilde{\bbz^j})\notin \cC_{j+1}}\nonumber
\end{align}
where the fourth inquality follows from the definition of $\cG_j$. 

Notice that $\bbx$ and $\tilde{\bbz}^j$ are independent. By the $(\eta, \tau, \nu)$-typical stability of $\adv_{j+1}$, for any fixed prefix $\bz^{j}$ we have
$$\pr{}{\adv_{j+1}(\bbx, \bz^j)\approx_{\eta, \tau}\alg_{j+1}(\dist)}\geq 1-\nu$$
Thus, by the independence of $\bbx$ and $\tilde{\bbz}^j$, we have
\begin{align}
\pr{}{(\bbx, \tilde{\bbz^j})\notin \cC_{j+1}}&\leq \nu.\label{ineq-last3} 
\end{align}

Putting (\ref{ineq-last1}), (\ref{ineq-last2}), and (\ref{ineq-last3}) together concludes the proof.
\end{proof}

The proof of Theorem~\ref{thm:composition} now follows from the Lemmas \ref{lem:basic2}, \ref{lem:typ-stable2}, and \ref{lem:ind-step2}, and via induction on the basis of $j=1$. Note that $\pr{}{(\bbx, \bbz^{0})\in\hH_1}=\pr{}{(\bbx, \bot)\in\cC_1}\geq 1-\nu$ by the typical stability of $\adv_1$, and hence the premise in Lemma~\ref{lem:basic2} is true for $j=1$ (the base case of the induction). 

\end{appendices}
\end{document}